\documentclass[sigconf]{acmart}

\AtBeginDocument{%
  \providecommand\BibTeX{{%
    \normalfont B\kern-0.5em{\scshape i\kern-0.25em b}\kern-0.8em\TeX}}}

\usepackage{algorithm}
\usepackage{algorithmic}
\usepackage{amsthm}

\newcommand{\bx}{\mathbf{x}}
\newcommand{\bv}{\mathbf{v}}
\newcommand{\bt}{\mathbf{t}}

\newcommand{\V}{\mathbf{V}}
\newcommand{\pr}{\text{Pr}}
\newcommand{\bW}{\mathbf{W}}
\newcommand{\ind}{\perp\!\!\!\perp}
\newtheorem{proposition}{Proposition}
\usepackage{color}
\usepackage{xcolor}
\usepackage{multirow}
\usepackage{booktabs}
\usepackage{amsmath}
\usepackage{bbm}
\usepackage{enumitem}
\usepackage{subcaption}
\usepackage{pifont}
\usepackage{xspace}

\newcommand{\method}{\texttt{SurvTRACE}\xspace}

\copyrightyear{2022} 
\acmYear{2022} 
\setcopyright{acmcopyright}\acmConference[BCB '22]{13th ACM International Conference on Bioinformatics, Computational Biology and Health Informatics}{August 7--10, 2022}{Northbrook, IL, USA}
\acmBooktitle{13th ACM International Conference on Bioinformatics, Computational Biology and Health Informatics (BCB '22), August 7--10, 2022, Northbrook, IL, USA}
\acmPrice{15.00}
\acmDOI{10.1145/3535508.3545521}
\acmISBN{978-1-4503-9386-7/22/08}

\begin{document}

%%
%% The "title" command has an optional parameter,
%% allowing the author to define a "short title" to be used in page headers.
\title{\method: Transformers for Survival Analysis with Competing Events}

%%
%% The "author" command and its associated commands are used to define
%% the authors and their affiliations.
%% Of note is the shared affiliation of the first two authors, and the
%% "authornote" and "authornotemark" commands
%% used to denote shared contribution to the research.

\author{Zifeng Wang}
\email{zifengw2@illinois.edu}
\affiliation{%
  \institution{UIUC}
  \city{Urbana}
  \state{IL}
  \country{USA}
}

\author{Jimeng Sun}
\email{jimeng@illinois.edu}
\affiliation{%
  \institution{UIUC}
  \city{Urbana}
  \state{IL}
  \country{USA}
}

% \author{Valerie B\'eranger}
% \affiliation{%
%   \institution{Inria Paris-Rocquencourt}
%   \city{Rocquencourt}
%   \country{France}
% }

% \author{Aparna Patel}
% \affiliation{%
%  \institution{Rajiv Gandhi University}
%  \streetaddress{Rono-Hills}
%  \city{Doimukh}
%  \state{Arunachal Pradesh}
%  \country{India}}

% \author{Huifen Chan}
% \affiliation{%
%   \institution{Tsinghua University}
%   \streetaddress{30 Shuangqing Rd}
%   \city{Haidian Qu}
%   \state{Beijing Shi}
%   \country{China}}

% \author{Charles Palmer}
% \affiliation{%
%   \institution{Palmer Research Laboratories}
%   \streetaddress{8600 Datapoint Drive}
%   \city{San Antonio}
%   \state{Texas}
%   \country{USA}
%   \postcode{78229}}
% \email{cpalmer@prl.com}

% \author{John Smith}
% \affiliation{%
%   \institution{The Th{\o}rv{\"a}ld Group}
%   \streetaddress{1 Th{\o}rv{\"a}ld Circle}
%   \city{Hekla}
%   \country{Iceland}}
% \email{jsmith@affiliation.org}

% \author{Julius P. Kumquat}
% \affiliation{%
%   \institution{The Kumquat Consortium}
%   \city{New York}
%   \country{USA}}
% \email{jpkumquat@consortium.net}

%%
%% By default, the full list of authors will be used in the page
%% headers. Often, this list is too long, and will overlap
%% other information printed in the page headers. This command allows
%% the author to define a more concise list
%% of authors' names for this purpose.
\renewcommand{\shortauthors}{Wang and Sun}

%%
%% The abstract is a short summary of the work to be presented in the
%% article.
\begin{abstract}
In medicine, survival analysis studies the time duration to events of interest such as mortality. One major challenge is how to deal with multiple competing events (e.g., multiple disease diagnoses). In this work, we propose a transformer-based model that does not make the assumption for the underlying survival distribution and is capable of handling competing events, namely \method. We account for the implicit \emph{confounders} in the observational setting in multi-events scenarios, which causes selection bias as the predicted survival probability is influenced by irrelevant factors. To sufficiently utilize the survival data to train transformers from scratch, multiple auxiliary tasks are designed for multi-task learning. The model hence learns a strong shared representation from all these tasks and in turn serves for better survival analysis. We further demonstrate how to inspect the covariate relevance and importance through interpretable attention mechanisms of \method, which suffices to great potential in enhancing clinical trial design and new treatment development. Experiments on METABRIC, SUPPORT, and SEER data with 470k patients validate the all-around superiority of our method. Software is available at \url{https://github.com/RyanWangZf/SurvTRACE}.
\end{abstract}

%%
%% The code below is generated by the tool at http://dl.acm.org/ccs.cfm.
%% Please copy and paste the code instead of the example below.
%%


%%
%% Keywords. The author(s) should pick words that accurately describe
%% the work being presented. Separate the keywords with commas.

%% A "teaser" image appears between the author and affiliation
%% information and the body of the document, and typically spans the
%% page.
% \begin{teaserfigure}
%   \includegraphics[width=\textwidth]{sampleteaser}
%   \caption{Seattle Mariners at Spring Training, 2010.}
%   \Description{Enjoying the baseball game from the third-base
%   seats. Ichiro Suzuki preparing to bat.}
%   \label{fig:teaser}
% \end{teaserfigure}

%%
%% This command processes the author and affiliation and title
%% information and builds the first part of the formatted document.
\maketitle

\section{Introduction}
Time-to-event analysis, or survival analysis, studies the \emph{probability} of event occurrence and the \emph{timing} of the event with broad applications, including medicine  \cite{friedman2015fundamentals}, reliability engineering \cite{pena2004models}, business analysis \cite{morrison2004introduction}. It also allows us to handle \emph{censored data}, i.e., we do not observe the occurrence of events due to the early stop of follow-ups, which has to be removed if we opt to use standard regression models.

\begin{figure}[t]
    \centering
    \includegraphics[width=0.45\textwidth]{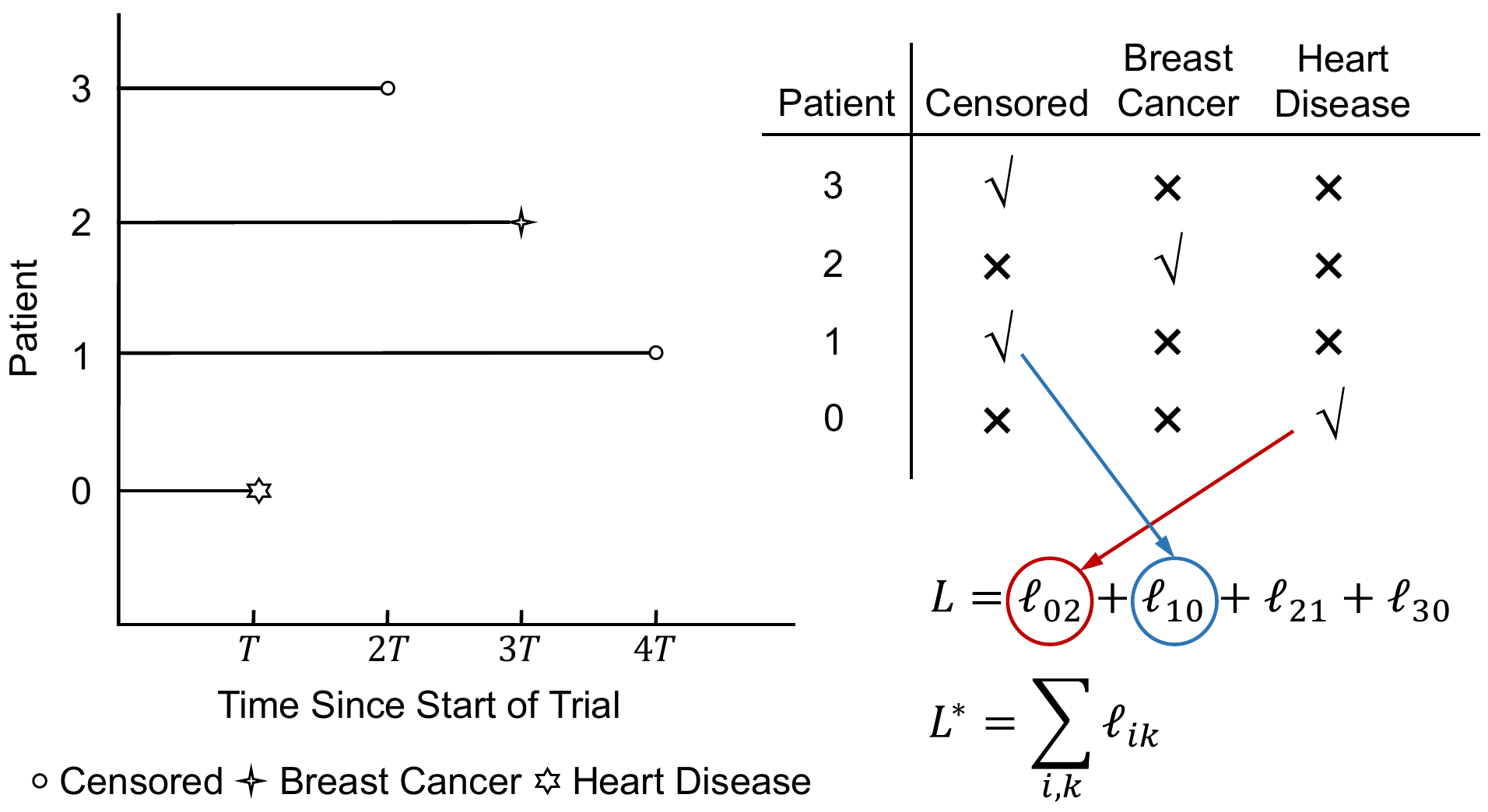}
    \caption{\textbf{Left}: Follow-up experience of 4 patients. \textbf{Right}: Patient-event table that shows event/censoring occurrence. \ding{51} indicates observed and \ding{55} is unobserved. Only the observed events can be used for training with the naive loss $L$ and the true loss $L^*$ is inaccessible.
    }
    \label{fig:bias_sample}
    % \vskip -1.5em
\end{figure}

Apart from data censoring, there are often multiple \emph{competing risks} which can lead to the incidence of events. For instance, it is prevalent that the elder who carries malignant cancer also suffers from other chronic diseases like diabetes. Dealing with competing risks is more difficult than with a single event, thus much less explored in the literature. These works try to alleviate the performance loss due to the strong assumption that each event is independent. However, the \emph{selection bias} in competing events survival data was seldom mentioned. Consider the targets $S(0)$ and $S(1)$ indicate the predicted survival probability for event 0 and event 1, respectively; $E \in \{0,1\}$ means which event is actually observed; and $X$ is the measured covariates of patients. Our estimated survival probability is unbiased only when $\{S(0),S(1)\} \ind E$, i.e. exchangeable \cite{little2019statistical}. In the observational setting of survival data, as shown by Fig. \ref{fig:bias_sample}, once an event was observed, others become \emph{counterfactuals}. In this case, the model performance is influenced by the event incidence and is biased towards those events which happen more frequently, or \emph{common events}.  Likewise, the model will have much worse performance on \emph{rare events}. In this sense, it is imperative to debias survival models such that it satisfies ignorability, i.e. $\{S(0),S(1)\} \ind E \mid X$, by counterfactual learning \cite{schnabel2016recommendations}.

On the other hand, deep learning (DL) was proved useful in enhancing survival analysis recently \cite{luck2017deep,lee2018deephit,nagpal2021deep}. However, they still suffer from insufficient training over rare events \cite{castaneda2010appraisal}. One reason is these models only use the basic multi-layer perceptron (MLP) with handcrafted features. The other challenge is that the current public survival data is either too small or too imbalanced to train strong DL models. That is, how to leverage strong DL models to boost survival analysis with limited survival data remains a challenge.

% \js{Try to shorten the past 3 paragraphs. Get to the point quickly. Right now it is convoluted without a clear point. Follow this skeleton: 1st paragraph - survival analysis is important; 2nd paragraph - what has been done in survival analysis and the gaps.}

In this work, we propose \method, which stands for \textbf{Surv}ival analysis using \textbf{TRA}ansformers with \textbf{C}ompeting \textbf{E}vents. \method is enabled by the following technical contributions:
%  \js{Try to correspond the early paragraphs to these contributions. Each contribution should map to a gap in the literature. If they are the contributions, these should be 1) novel, 2) challenging and 3) useful. We need a response to this comment: ``This paper just applied transformer and multi-task learning for survival analysis. The approach seems incremental and technical contribution is limited''}
\begin{enumerate}
    \item \textbf{Debiasing competing events analysis using counterfactual learning.} We develop a learning method based on inverse propensity score (IPS) \cite{little2019statistical} that remedies selection biases in survival data with competing risks. This method guarantees unbiased evaluation and learning of survival analysis models hence outperforms methods that ignore selection bias,  especially on the analysis for rare events.
    \item \textbf{Automatic feature engineering with attentive encoders.} We study how to automatically learn high-order interactions between covariates through attention \cite{devlin2019bert}, to avoid manual feature engineering. Meanwhile, we inspect how the learned attention scores demonstrate relevance between covariates as well as show interpretability for the prediction results.
    \item \textbf{Multi-task learning with a shared backbone.} We design multiple auxiliary tasks to make the best of limited survival data to train \method from scratch. A shared representation is learned from all tasks then serves as a strong backbone for downstream tasks. This fashion strengthens the prediction accuracy for all events by sharing common knowledge.
\end{enumerate}

We evaluate \method on two open survival datasets and a large-scale dataset with 470k patients, where it outperforms the state-of-the-art baselines significantly.

\begin{figure*}[t]
\centering
\includegraphics[width=0.9\textwidth]{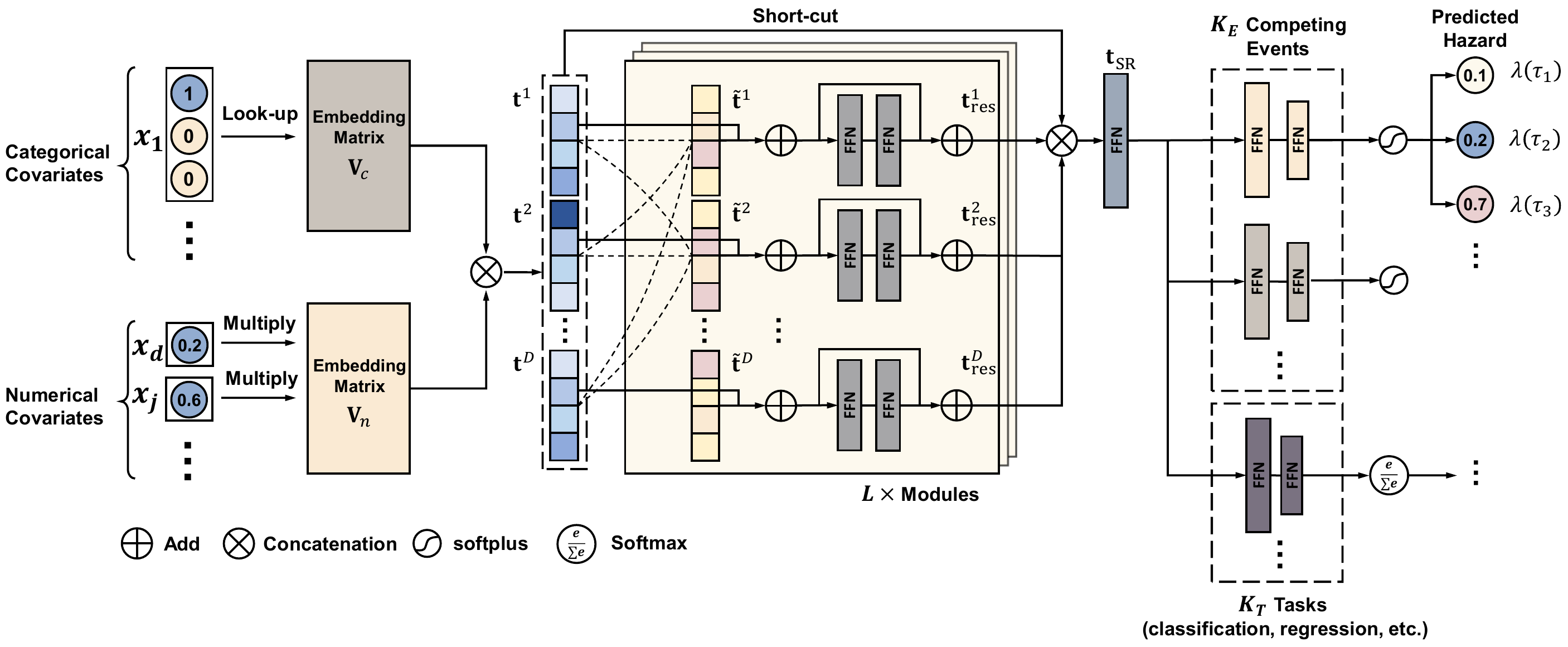}
\caption{\method architecture. The raw numerical and discrete covariates are encoded through two embedding matrices $\mathbf{V}_c$ and $\mathbf{V}_n$ separately. Two types of embeddings ($\mathbf{t}^1,\dots,\mathbf{t}^D$) are concatenated before going into the attentive encoder layers, where covariate embeddings interact to form high-order combinatorial embeddings $\widetilde{\mathbf{t}}^1,\dots,\widetilde{\mathbf{t}}^D$. The yielded representations $\mathbf{t}_{\text{SR}}$ are shared across all task-specific sub-networks for downstream tasks. For survival analysis, hazard ratio $\lambda(\tau)$ of each discrete duration index $\tau$ is predicted. \label{fig:architecture}}
% \vskip -1.5em
\end{figure*}

\section{Related Work}
The thriving need for survival analysis encourages a plethora of statistical methods. For instance, the Cox proportional hazards model (CPH) \cite{cox1972regression} which is multivariate linear regression. To enhance CPH, multi-task learning \cite{li2016multi}, transfer learning \cite{li2016transfer}, active learning \cite{vinzamuri2014active} were used. On the other hand, new models advanced by machine learning (ML), e.g., survival support vector machine \cite{van2007support,polsterl2015fast}, random survival forests (RSF) \cite{ishwaran2008random}, gradient boosting \cite{wang2020boxhed}, were proposed. There were also attempts using neural networks for learning representations of covariates \cite{luck2017deep, katzman2018deep,lee2018deephit, ren2019deep, nagpal2021deep}. However, these NN-based methods do not fully exploit the power of NNs as they only use simple multi-layer perceptron, which is inherently limited in its learning capacity. More importantly, few of them are interpretable such that it is unclear what the black-box model learns and how we gain insight from the predictions for applications, e.g., identifying risk factors \cite{koene2016shared} and guiding design of clinical trials \cite{liu2021evaluating}. Readers may refer to \cite{wang2019machine} for a survey of ML-based survival analysis. 

Transformers were proposed by \citet{vaswani2017attention} for machine translation and have since been applied to extensive applications in natural language processing \cite{devlin2019bert,liu2019roberta,wang2022trial2vec}, computer vision \cite{dosovitskiy2020image,liu2021swin}, and data mining \cite{gorishniy2021revisiting,wang2022transtab}. Compared to these domains, it was much less used for survival analysis. In survival analysis, most related to ours are BERTSurv \cite{zhao2021bertsurv} and TDSA \cite{hu2021transformer} which both leverage transformers. BERTSurv extracts the word embeddings of clinical notes by BERT and combines them with other measurement features then feed them to an MLP for survival predictions, which is still an application of transformers for texts and can not deal with purely tabular data analysis. TDSA takes a single MLP layer to project all the input features into low-dimensional representations and differentiates these embeddings by adding time-aware positional embeddings. The self-attention interactions are taken between the patient embedding and the time. This approach ignores the interaction between features and relies on the manual settings of time embeddings, which results in the inapparent superiority of TDSA over the plain MLP-based baselines. However, \method encodes each feature in a low-dimensional embedding and takes full interactions between features with self-attention. Besides, \method applies to competing events and proves the effectiveness of transformers for large-scale survival analysis.

%  Different from the above, our model utilizes the cutting-edge transformers as the backbone that also provides the means of interpretability.

In competing events scenarios, many works assume each event is independent and handle each event separately by setting others to be censored \cite{ranganath2016deep,chapfuwa2018adversarial,nagpal2021deep}. The existing competing event analysis methods \cite{lee2018deephit,bellot2018multitask, bellot2018tree, lee2019dynamic,  tjandra2021hierarchical, rahman2021deeppseudo} try to weaken the event independent assumption but do not consider bias and imbalance in survival data. Selection bias in multi-event data, shown by Fig. \ref{fig:bias_sample}, causes naive multi-event loss a biased estimate of the true loss when the occurrences of events are covariate-dependent and the rare events are under-represented \cite{wang2020information}. This bias was hardly discussed in the survival analysis literature, which makes the most difference of our work from the others.

\section{Model Architecture}
In this section, we elaborate on the architecture and inference of \method.

\subsection{Problem Formulation}\label{sec:problem_formulation}
We assume our survival data consists of three parts $\mathcal{D}=\{(\bx_i, t_i, e_i)\}_{i=1}^n$, where $\bx_i \in \mathbb{R}^D$ is $D$ baseline covariates associated with the $i$-th patient; $t_i \in \mathbb{R}$ is the time at which an event of interest takes place or the time when the sample is censored; and $e_i$ shows whether $t_i$ is the event occurance or censored time. For single event scenarios, $e_i$ is a binary indicator; For competing events scenarios, $e_i \in \{0, 1, \dots, K_E \}$ tells which event happens at time $t_i$. When $e_i = 0$, the patient is said to be \emph{right-censored} (i.e., no event is recorded at the end of the study for patient $i$). %In the right-censoring situation, if one opts to use standard regression models, these censored samples must be discarded.

The goal of survival analysis is to estimate the hazard and survival function. The survival function signifies the alive probability for one patient at time $t$, as $S(t) \triangleq \pr(T > t)$. The hazard function is defined by
\begin{equation} \label{eq:limit_hazard}
    \lambda (t) \triangleq \lim_{\Delta t \to 0} \frac{\pr(t \leq T < t+\Delta t \mid T \geq t)}{\Delta t},
\end{equation}
which corresponds to the probability of death at time $t$ given that the patient has survived up to that point. Likewise, we denote the probability mass function (PMF) of event time by $g(t) = \pr(T=t)$.

\subsection{\method: Main Architecture}
As illustrated in Fig. \ref{fig:architecture}, \method includes a baseline covariates embedding module, a deep-stacked attentive encoder module, and the alignment and subnetwork prediction module. Next we will present the technical details of these modules and then introduce how our model is used for training and inference.

\subsubsection{Input \& Embedding Module}
The raw baseline covariates describe the characteristics of patients. These covariates can be separated into two types:  categorical and numerical. We set $D_c$ and $D_n$ as the number of categorical and numerical covariates, respectively. We denote the number of covariates as $D = D_c + D_n$. Categorical covariates are usually transformed into one-hot vectors before entering into the survival model. Here, we represent them in a $d_e$-dimensional space through an embedding matrix $\V_{c} \in \mathbb{R}^{D_c \times d_e}$ as 
\begin{equation}
    \bt_i^m = \V_{c} \bx_i^m,
\end{equation}
where $\bx_i \in \mathbb{R}^{D_c}$ is a one-hot vector for the category field $m$ and $\bt_i^m$ is the yielded embedding.

To allow interactions between numerical and categorical covariates, we represent numerical covariates in a low-dimensional space by
\begin{equation}
    \bt^j_i = \bv_j \bx_i^j,
\end{equation}
where $\bx_i^j$ is a scalar for the $j$-th numerical field, $\V_n \in \mathbb{R}^{D_n \times d_e}$ is the embedding matrix for numerical features and $\bv_j$ is the $j$-th row of $\V_n$. With both two types of embeddings at hand, we can concatenate them to obtain the representation $\bt_i$ for all raw input covariates of the $i$-th patient, such that
\begin{equation}
    \bt_i = \bt_i^1 \otimes \bt_i^2 \cdots \otimes \bt_i^D,
\end{equation}
where $\otimes$ denotes concatenation operation.

\subsubsection{Encoder Module} To enable sufficient interactions between covariate embeddings, we use multi-head self-attention. The key-value attention mechanism allows the model to learn combinatorial interactions automatically. 

In detail, the key idea is to obtain the each $j$-th field processed covariate embedding $\widetilde{\bt}^j_i$ through a combination of other embeddings weighted by their correlation, as
\begin{equation}\label{eq:single_head}
        \widetilde{\bt}_i^{j} = \sum_{k=1}^D \alpha_{j,k} (\bW_{\text{value}} \bt_i^k),
\end{equation}
where $\alpha_{j,k}$ signifies the the relevance score between covariate $j$ and $k$; $\bW_{\text{value}}$ is a learnable weight matrix to transform the raw embeddings to the same space. We omit the subscript $i$ of $\bt$ for avoiding clutter notations. $\alpha_{j,k}$ is produced by the softmax outputs of attention function $\psi$ as
\begin{equation}
    \alpha_{j,k} = \frac{\exp(\psi(\bt^j,\bt^k))}{\sum_{k^\prime=1}^{D} \exp(\psi(\bt^j,\bt^{k^\prime}))},
\end{equation}
where the attention function $\psi(\cdot,\cdot)$ could be an arbitrary function that maps two embeddings to a real-value output. Here, we set it as
\begin{equation}
    \psi(\bt^j,\bt^k) = \langle \bW_{\text{query}} \bt^j, \bW_{\text{key}} \bt^k \rangle,
\end{equation}
which is an inner product of the transformed embedding by two weight matrices: $\bW_{\text{query}}$ and $\bW_{\text{key}}$. 
%$\alpha_{j,k}$ can also work for visualizing the relevance between covariate $j$ and $k$, which we will demonstrate in experiments. 

The attention in Eq. \eqref{eq:single_head} can be further enhanced by introducing $H$ multi-heads to yield a series of diverse processed embeddings $\{\widetilde{\bt}^{1,j} \dots \widetilde{\bt}^{H,j}\}$. Then the embedding for the field $j$ is obtained by concatenation of $\widetilde{\bt}^{h,j}$ from all heads
\begin{equation}
    \widetilde{\bt}^j = \widetilde{\bt}^{1,j}\otimes  \widetilde{\bt}^{1,j} \cdots \otimes \widetilde{\bt}^{H,j}.
\end{equation}
Likewise, we have $H$ pairs of attention parameters as $\bW_{\text{query}}^h$, $ \bW_{\text{value}}^h$, $\bW_{\text{key}}^h \forall h=1 \dots H$.

To reserve the information of the raw input embeddings, a residual connection is added to yield the final embedding
\begin{equation}
    \bt^j_{\text{res}} = \text{SELU}(\bW_{\text{res}} \widetilde{\bt}^j + \bt^j).
\end{equation}
$\text{SELU}(\cdot)$ denotes the Scaled Exponential Linear Unit (SELU) activation function \cite{klambauer2017self}. The output embedding of the transformer layer is then obtained by another $l_1$-layer feed-forward network (FFN) with residual connection
\begin{equation}
    \hat{\bt}^j = \text{SELU}(\bW^{(l_1)}_{\text{FFN}}( \cdots \bW^{(2)}_{\text{FFN}}( \text{SELU}(\bW^{(1)}_{\text{FFN}} \bt^j))) + \bt^j_{\text{res}}).
\end{equation}

In a nutshell, from the first transformer, we transform the raw embedding $\bt^j$ to the attentive embedding $\hat{\bt}^j$. To encourage further interactions between covariates to get high-order combinatorial embeddings, we can stack $l_2$ transformers such that
\begin{equation}
    \bt^j \to \hat{\bt}^{(1,j)} \to \cdots \hat{\bt}^{(l_2,j)},
\end{equation}
hence $\hat{\bt} = \hat{\bt}^{(l_2,1)} \otimes \cdots \otimes \hat{\bt}^{(l_2,D)}$ is the final representation for the patient generated by the stacked transformer encoder.

\subsubsection{Shared Representation \& Sub-networks Module}
\method builds a shared representation $\bt_{\text{SR}}$ from the encoder and serves for all downstream tasks, which enables the model to learn generalizable knowledge across all tasks and gets better performance for each task. Upon obtaining the representation $\hat{\bt}$, we design a shortcut to concatenate it with $\bt$, therefore utilize an alignment layer to transform them to the same space
\begin{equation}
    \bt_{\text{SR}} = \text{SELU}(\bW_{\text{SR}} (\hat{\bt} \otimes \bt)).
\end{equation}
Based on the shared representation $\bt_{\text{SR}}$, we can design many sub-networks for downstream tasks. These tasks can be split into two buckets: major tasks (survival analysis) and auxiliary tasks. Next, we will elaborate on how to deal with these tasks with different sub-network designs.

\section{Task Design for Learning}
\subsection{Task I: Single-Event Survival Analysis}
The ultimate goal of survival analysis is to estimate the survival function for individual patients, which is the PMF of survival time distribution. To make  continuous-time hazard rate prediction feasible for neural networks, we parameterize the discrete-time hazard rate of events by a sub-network. Consider a time point set $\mathbb{T} = \{\tau_1, \dots, \tau_m \}$ where $\tau_{m}$ is the pre-defined maximum follow-up time horizon. The discrete index set $\kappa(t) = \{1,\dots,m \}$. In this scenario, the hazard rate at time $\tau_j$ is defined by
\begin{equation}
    \lambda(\tau_j) = \pr(T = \tau_j \mid T > \tau_{j-1}),
\end{equation}
hence each corresponds to an output node of the sub-network. Likewise, we write the censored time as $T_C \in \mathbb{T}$.

First, let us take the single-event ($e\in \{0,1\}$) survival analysis as the case. Assuming $T$ and $T_C$ are independent, we can write the likelihood function  by 
\begin{equation}\label{eq:joint_likelihood}
    \begin{aligned}
    \pr(T=t,E=e) & = [\pr(T=t)\pr(T_C\geq t)]^e \\
     & \times [\pr(T>t) \pr(T_C = t)]^{1-e}.
\end{aligned}
\end{equation}
We can omit the terms which are only determined by censored time distribution (e.g., the PMF of censored time), then denote $S(t)$ and $g(t)$ by discrete hazard rate $\lambda(t)$ as
\begin{equation}
    g(\tau_j) = \lambda(\tau_j)S(\tau_{j-1}), \ S(\tau_j) = [1-\lambda(\tau_j)]S(\tau_{j-1}).
\end{equation}
With the assumption of constant hazard within each interval $[\tau_{j-1},\tau_j]$, the piecewise constant hazard (PCH) loss \cite{kvamme2019continuous} is defined by
\begin{equation} \label{eq:single_pch_loss}
    \ell_i = -e_i \log \lambda(t_i) +\lambda(t_i) \rho(t_i) +\sum_{j=1}^{\kappa(t_i)-1}\exp[-\lambda(\tau_j)],
\end{equation}
where $\rho(t)$ is the proportion of interval $\kappa(t)$ as time t as $\rho(t) = (t-\tau_{\kappa(t)-1}) / (\tau_{\kappa(t)} - \tau_{\kappa(t)-1})$. With the shared representation $\bt_{\text{SR}}$, the output network outputs the hazard rate prediction as
\begin{equation}
    \lambda(t) = \log \left [ 1 + \exp [f(\kappa(t)|\bt_{\text{SR}})] \right ],
\end{equation}
which is used for training by PCH loss in Eq. \eqref{eq:single_pch_loss}.

\begin{table*}[t]
  \caption{Descriptive statistics of datasets. BC and HD are shorthands of breast cancer and heart diseases, respectively.}

  \centering
  \renewcommand\arraystretch{1.2}
    \begin{tabular}{|c|c|c|c|c|c|c|c|c|c|c|}
    \hline
    \multicolumn{2}{|c|}{\multirow{2}{*}{\textbf{Dataset}}} & \multirow{2}{*}{\textbf{No. Events}} & \multirow{2}{*}{\textbf{No. Censored}} & \textbf{No. Covariates} & \multicolumn{3}{c|}{\textbf{Event Duration}} & \multicolumn{3}{c|}{\textbf{Censoring Time}} \\
\cline{5-11}    \multicolumn{2}{|c|}{} &       &       & (real, categorical) & min   & max   & mean  & min   & max   & mean \\
    \hline
    \multicolumn{2}{|c|}{\textbf{METABRIC}} & 1,103 (57.9\%) & 801 (42.1\%) & 9 (5, 4) & 0.1   & 355.2 & 99.9  & 0     & 337   & 159.5 \\
\cline{1-2}    \multicolumn{2}{|c|}{\textbf{SUPPORT}} & 6,036 (68.0\%) & 2,837 (32.0\%) & 14 (8, 6) & 3     & 1944  & 205.4 & 344   & 2029  & 1059.8 \\
    \hline
    \multirow{2}{*}{\textbf{SEER}} & \textbf{BC}    & 87,495 (18.4\%) & \multirow{2}{*}{367,702 (77.1\%)} & \multirow{2}{*}{18 (4, 14)} & 1     & 121   & 40.2  & \multirow{2}{*}{1} & \multirow{2}{*}{121} & \multirow{2}{*}{74.7} \\
          & \textbf{HD}    & 21,549 (4.5\%) &       &       & 1     & 121   & 53.4  &       &       &  \\
    \hline
    \end{tabular}%

  \label{tab:data}%
\end{table*}%

\subsection{Task II: Debiasing Competing Events Survival Analysis}
In competing events scenarios, $e$ is no longer a binary indicator. Instead, we have $K_E$ competing events as $e \in \{1,\dots, K_E\}$. Denote $\mathbbm{1}_{ik}=\mathbbm{1}\{e_i=k\}$ and $\ell_{ik}$ is $\ell_i$ given $e_i=k$, a naive adaptation from Eq. \eqref{eq:single_pch_loss} for competing events is to take 
\begin{equation}
    L_{\text{naive}} \triangleq \frac1{\sum_{i,k}\mathbbm{1}_{ik}} \sum_{i,k} \mathbbm{1}_{ik} \ell_{ik}.
\end{equation}
The hazard prediction is performed by the attached cause-specific (CS) sub-networks as $\lambda_k(t) = \log [1 + \exp [f_k(\kappa(t)|\bt_{\text{SR}})] ]$ for $k=1,\dots, K_E$.

However, $ L_{\text{naive}}$ assumes events are independent but in reality ar often biased by the common events. This bias exaggerates with more imbalanced event distribution, which causes poor performance. Unfortunately, imbalanced event distribution is common in the real world. For instance, the $15\%$ events in SEER data used \cite{lee2018deephit} are breast cancer, while only $1\%$ are cardiovascular diseases. To resolve it, we leverage the inverse propensity score (IPS) technique for debiasing. Denote $\pi_{ik} = \pr(e_i=k|\phi,\bx)$ as the estimate of $\pr(e_i=k)$, a.k.a propensity score, we derive a novel IPS-based PCH loss as
\begin{equation}
      L_{\text{IPS}} \triangleq \frac1{nK_E} \sum_{i,k} \frac{ \mathbbm{1}_{ik} \ell_{ik}}{\pi_{ik}}.  
\end{equation}
Here, $\phi(\bx)$ is a logistic regression model
\begin{equation} \label{eq:IPS_estimator}
    \pi_{ik} = \phi(\bx_i) \triangleq \sigma(\mathbf{w}^{\top}\bx_i + \beta),
\end{equation}
where $\beta$ denotes the offset; $\sigma(\cdot)$ is sigmoid function. Please refer to Appendix \ref{sec:appx_lh_ips} for the proof of why $ L_{\text{IPS}}$ is unbiased with further explanation.

\subsection{Auxiliary Tasks: Multi-task Learning}
Multi-task Learning (MTL) puts the model to learn from multiple related tasks with shared representations, thus enabling the model to generalize better on the targeted task. Besides survival analysis, we design two auxiliary tasks for enhancing the representation learning: mortality prediction (MP) and length-of-stay prediction (LS). 

For the mortality prediction task, we urge the model to learn to predict if there will be an event happening ($\delta = 1$ if $e>0$) during the whole time
\begin{equation}
    L_{\text{MP}} = -\frac1n \sum_i \left[\delta_i \log \hat{y}_i + (1-\delta_i) \log (1 - \hat{y}_i) \right],
\end{equation}
where $\hat{y}_i$ is predicted by the task-specific (TS) sub-network.

For the length-of-stay prediction task, the model predicts how much time the event happens or becomes censored after the initial observation
\begin{equation}
    L_{\text{LS}} = \frac1n \sum_i (\hat{t}_i - t_i)^2,
\end{equation}
where $\hat{t}_i$ is the predicted event time. Afterwards, we can write the final loss function by
\begin{equation}
    \mathcal{L} = L_{\text{IPS}} + \gamma_1 L_{\text{MP}} + \gamma_2 L_{\text{LS}}.
\end{equation}
Two hyperparameters $\gamma_1$ and $\gamma_2$ can be set to 1 initially and then annealed in the following training.

% Table generated by Excel2LaTeX from sheet 'Sheet2'
\begin{table*}[t]
  \caption{$C^{\text{td}}$ for \textbf{METABRIC} and \textbf{SUPPORT} datasets at different quantiles of event times; Values in the bracket show the standard deviation of performances of 10 runs.}

\renewcommand\arraystretch{1.2}
  \centering
    \begin{tabular}{|c|c|c|c|c|c|c|}
    \hline
    \multirow{2}{*}{\textbf{Algorithms}} & \multicolumn{3}{c|}{\textbf{METABRIC}} & \multicolumn{3}{c|}{\textbf{SUPPORT}} \\
\cline{2-7}          & $25\%$  & $50\%$   & $75\%$  & $25\%$  & $50\%$   & $75\%$ \\
    \hline
    CPH & 0.628(0.024) & 0.627(0.020) & 0.632(0.016) & 0.549(0.017) & 0.564(0.004) & 0.586(0.005) \\
    \hline
    DeepSurv & 0.660(0.028) & 0.648(0.022) & 0.644(0.018) & 0.594(0.013) & 0.591(0.007) & 0.605(0.006) \\
    \hline
    DeepHit & 0.712(0.026) & 0.657(0.023) & 0.603(0.014) & 0.650(0.009) & 0.602(0.013) & 0.574(0.009) \\
    \hline
    RSF & 0.698(0.029) & 0.658(0.022) & 0.630(0.017) & 0.660(0.005) & 0.621(0.006) & 0.602(0.006) \\
    \hline
    PC-Hazard & 0.713(0.024) & 0.680(0.017) & 0.644(0.017) & 0.652(0.011) & 0.620(0.008) & 0.607(0.008) \\
    \hline
    DSM   & 0.707(0.023) & 0.663(0.014) & 0.636(0.017) & 0.640(0.007) & 0.609(0.007) & 0.596(0.008) \\
    \hline \hline
    \method$_{\text{w/o MTL}}$ & 0.722(0.022) & 0.686(0.010) & 0.649(0.017) & 0.665(0.008) & 0.630(0.006) & 0.614(0.005) \\
    \hline
    \method & \textbf{0.728(0.019)} & \textbf{0.690(0.013)} & \textbf{0.655(0.013)} & \textbf{0.670(0.008)} & \textbf{0.633(0.006)} & \textbf{0.617(0.004)} \\
    \hline

    \end{tabular}%
  \label{tab:ctd_single}%
%   \vspace{-1em}
\end{table*}%

% Table generated by Excel2LaTeX from sheet 'Sheet2'
\begin{table*}[t]
  \caption{$C^{\text{td}}$ for competing risks on \textbf{SEER} dataset; Values in the bracket show the standard deviation of 10 runs.}
\renewcommand\arraystretch{1.2}
  \centering
    \begin{tabular}{|c|c|c|c|c|c|c|}
    \hline
    \multirow{2}{*}{\textbf{Algorithms}} & \multicolumn{2}{c|}{$25\%$} & \multicolumn{2}{c|}{$50\%$} & \multicolumn{2}{c|}{$75\%$} \\
\cline{2-7}          & HD    & Breast Cancer & HD     & Breast Cancer & HD     & Breast Cancer \\
    \hline
    CS-CPH & N/A   & 0.828(1.5e-3) & N/A   & 0.799(1.1e-3) & N/A   & 0.781(7e-4) \\
    \hline
    CS-PC-Hazard & 0.774(5.1e-3) & 0.895(1.7e-3) & 0.769(3.3e-3) & 0.875(1.6e-3) & 0.766(3.9e-3) & 0.858(4e-4) \\
    \hline
    DeepHit & 0.763(1.6e-2) & 0.896(2.2e-3) & 0.748(1.5e-2) & 0.875(2.7e-3) & 0.724(1.13-2) & 0.853(1.6e-3) \\
    \hline
    DSM   & 0.765(4.6e-3) & 0.895(1.2e-3) & 0.761(3.9e-3) & 0.873(2.0e-3) & 0.750(2.3e-3) & 0.856(1.2e-3) \\
    \hline \hline
    \method$_{\text{w/o IPS}}$ & 0.789(6.3e-3) & 0.902(1.2e-3) & 0.780(5.0e-3) & 0.882(1.3e-3) & 0.768(2.7e-3) & 0.864(9e-4) \\
    \hline
    \method$_{\text{w/o MTL}}$ & 0.793(6.4e-3) & 0.903(1.1e-3) & 0.784(5.4e-3) & 0.881(1.5e-3) & 0.768(3.1e-3) & 0.863(5e-4) \\
    \hline
    \method & \textbf{0.797(6.2e-3)} & \textbf{0.904(1.2e-3)} & \textbf{0.788(5.5e-3)} & \textbf{0.883(1.2e-3)} & \textbf{0.775(3.1e-3)} & \textbf{0.866(5e-4)} \\
    \hline
    \end{tabular}%
  \label{tab:ctd_compete}%
\end{table*}%

\section{Experiment}
In this section, we resort to focus on the following four research questions:
\begin{itemize}[leftmargin=*, itemsep=0pt, labelsep=5pt]
\item \textbf{RQ1.} Does high-order covariates interaction with transformers encourage better performance?
\item \textbf{RQ2.} How much does selection bias harm the competing events survival analysis?
\item \textbf{RQ3.} Does MTL help learn a stronger encoder of \method for survival analysis?
\item \textbf{RQ4.} What is the insight \method can offer by its interpretable function?
\end{itemize}
We will first present the setups then discuss each of the RQs one by one.

\subsection{Experimental Setup}
\subsubsection{Datasets}
For single event survival analysis, we evaluate models on two real-world medical datasets: Study to Understand Prognoses Preferences Outcomes and Risks of Treatment (\textbf{SUPPORT}) \cite{knaus1995support} and Molecular Taxonomy of Breast Cancer International Consortium (\textbf{METABRIC}) \cite{curtis2012genomic}. For competing events, we collect and proceed with the data from Surveillance, Epidemiology, and End Results Program (\textbf{SEER})\footnote{\url{https://seer.cancer.gov/}}.

\textbf{SUPPORT.} It is a multicenter study designed to examine outcomes and clinical decision-making for seriously ill hospitalized patients. The version we utilize comes from the pycox\footnote{\url{https://github.com/havakv/pycox}} package \cite{kvamme2019continuous} following the preprocessing steps in \cite{katzman2018deep}.

\textbf{METABRIC.} This data uses gene and protein expression profiles to determine new breast cancer subgroups to help physicians provide better treatment. We utilize the data from pycox as well.

\textbf{SEER.} This is an authoritative source for cancer statistics in the US. We select breast cancer patients registered from 2004 to 2014, with the follow-up period restricted to 10 years. Among all these patients, we select who also suffer from heart diseases, which renders a large-scale dataset with 476,746 patients. Therefore, we treat breast cancer and heart diseases as two competing events. We include 18 covariates, including age, race, gender, diagnostic confirmation, morphology information (primary site, laterality, histologic type, etc.), tumor information (size, type, number, etc.), and surgery information. We fill missing values with the mean of numerical covariates and mode of categorical covariates. The statistics of all datasets are available in Table \ref{tab:data}. 

\subsubsection{Evaluation Metrics}
We make use of time-dependent concordance index ($C^{\text{td}}$) \cite{antolini2005time} for the performance evaluation of $k$-th event, as
\begin{equation}
\begin{aligned}
    C^{\text{td}}(\tau,k) = \pr  \{S_k(\tau|\bx_i) & > S_k(\tau|\bx_j)  \mid e_i=k, \\
   & t_i < t_j, t_i \leq \tau, k>0  \}.
\end{aligned}
\end{equation}
Here, $S_k(t|\bx_i)$ is the predicted survival function considering the $k$-th event at the truncation time $\tau$. We adjust the estimate with an inverse probability of censoring weighted (IPCW) estimate to obtain an unbiased estimate following \cite{uno2011c}. Since $C^{\text{td}}$ at different time horizons indicate how models capture the possible changes in risk over time, we follow \cite{nagpal2021deep} to report $C^{\text{td}}$ at different truncated time quantiles of $25\%$, $50\%$, and $75\%$. 

\subsubsection{Baselines}
We pick the following baselines for comparison: Cox Proportional Hazards (CPH) \cite{cox1972regression}, Random Survival Forests (RSF) \cite{ishwaran2008random}, DeepSurv \cite{katzman2018deep}, DeepHit \cite{lee2018deephit}, Piecewise Constant Hazard (PC-Hazard) \cite{kvamme2019continuous}, and Deep Survival Machines (DSM) \cite{nagpal2021deep}.

For competing event survival analysis, we pick DeepHit and DSM, which apply to these cases. We also utilize cause-specific CPH (CS-CPH) and cause-specific PC-Hazard (CS-PC-Hazard) by assuming the independence of events \cite{haller2013applying}. These CS-based methods learn from each competing event separately by treating others as censored.

\begin{figure}[t]
  \begin{subfigure}[b]{0.23\textwidth}
    \includegraphics[width=\textwidth]{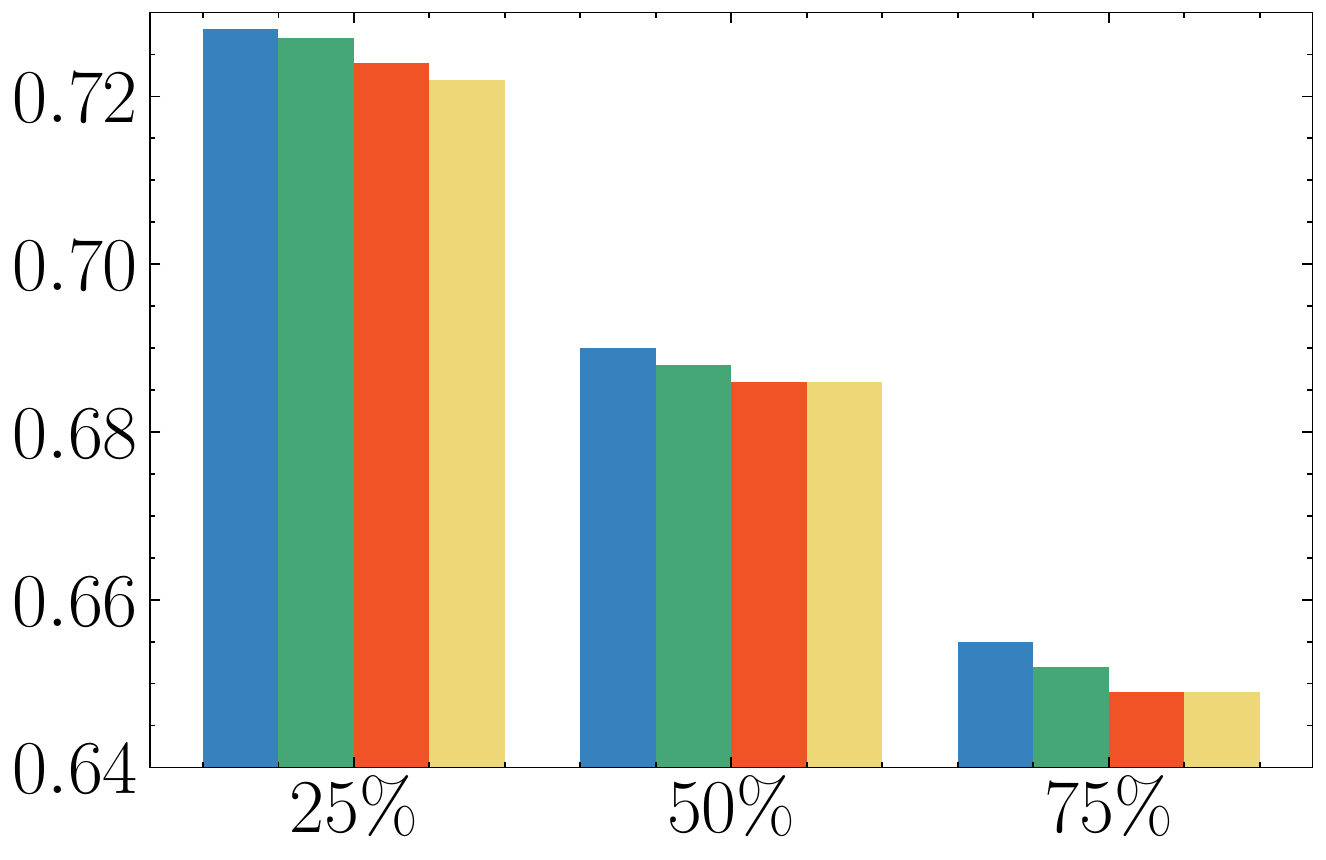}
    \caption{On \textbf{METABRIC}.}
    % \label{fig:mtl_ablation}
  \end{subfigure}
  \hfill
  \begin{subfigure}[b]{0.23\textwidth}
    \includegraphics[width=\textwidth]{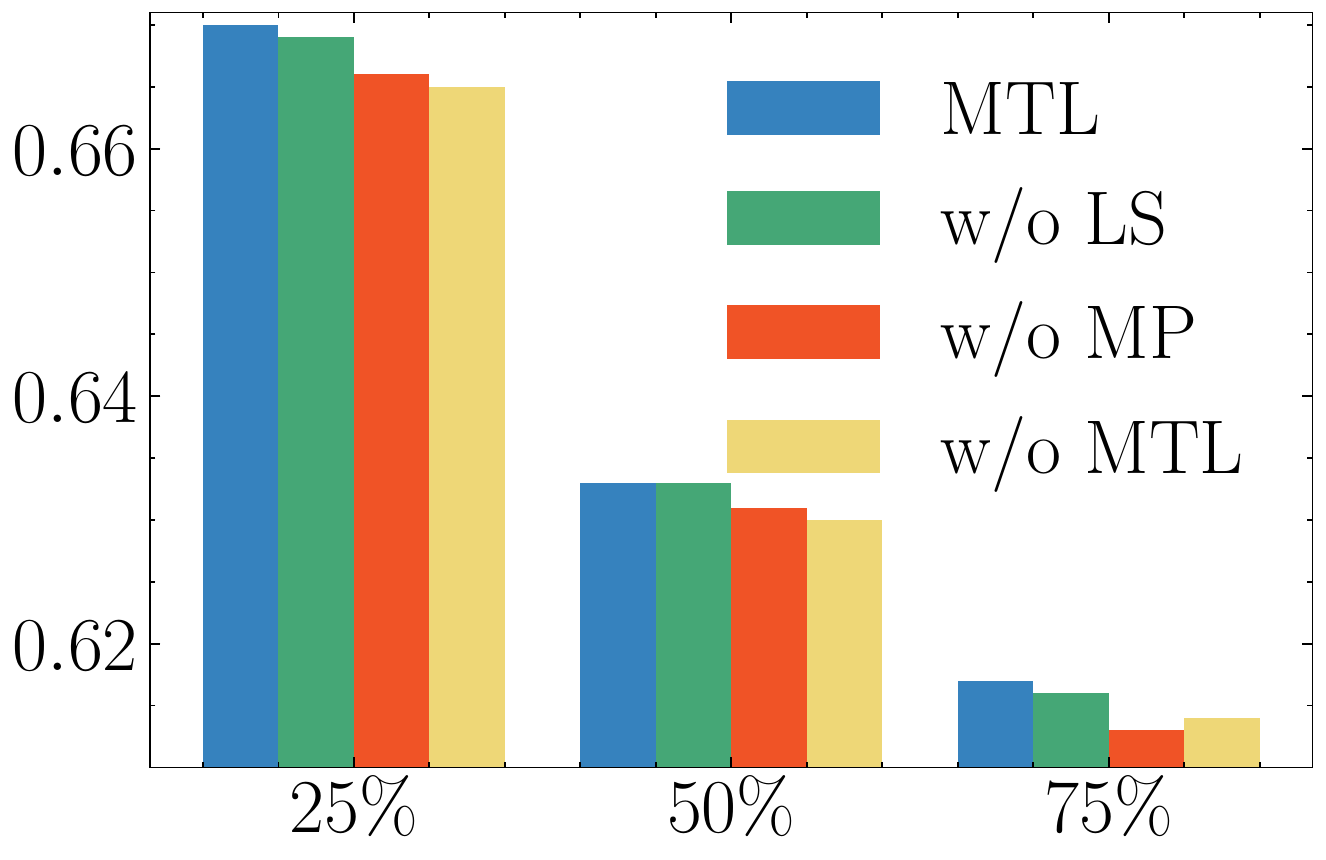}
    \caption{On \textbf{SUPPORT}.}
    % \label{fig:mtl_ablation}
  \end{subfigure}
  \caption{Ablation study of MTL for \method on \textbf{METABRIC} and \textbf{SUPPORT}. We compare the full MTL model with models without length-of-stay (LS) prediction, mortality prediction (MP), and without MTL, respectively. $x$-axis and $y$-axis are time horizons and $C^{\text{td}}$. \label{fig:mtl_ablation}}
%   \vskip -1.5em
\end{figure}

\subsubsection{Implementation}
We use Adam \cite{kingma2014adam} as the optimizer to train \method in all experiments, with learning rate in $\{1e^{-4}, 1e^{-3}\}$, weight decay in $\{1e^{-3},1e^{-4},0\}$. The number of transformer layers is chosen from $\{2,3,4\}$, the embedding size is selected from $\{8,16\}$, the intermediate layer size is picked from $\{32, 64\}$, and the number of attention heads is set from $\{1, 2, 4\}$. The cause-specific and task-specific sub-networks are MLPs with one or two layers with the same intermediate size as the transformers and ReLU activation.

We use $30\%$ data as the test set,  $10\%$ data as the validation set, and $60\%$ as the training set. We report the mean and standard deviation of metrics with 10 multiple runs on different train/validation/test splits.

\subsection{Performance Comparison in Single Event (RQ1)}
We compare \method with baselines on single event datasets. Results are reported in Table \ref{tab:ctd_single}, where the best-performing method is shown in bold. We find that \method without MTL consistently outperforms all baselines across two datasets in terms of $C^{\text{td}}$ under all time horizons. The reasons for this improvement are multi-facet: (1) \method leverages multi-head attention to build rich interactions among covariates and dynamically adjust to different inputs; (2) The stacked attentive encoders provide higher-order covariate conjunctions, thus mining more complicated patterns from the survival data. Furthermore, \method gets better results when engaged with MTL.

We also observe that all methods experienced performance deterioration on longer horizons like $75\%$. As more patients are involved in the evaluation, it is harder for models to predict the orders of event time of all patients. Nevertheless, \method performs better across for all horizons with the relative improvement of $2.1\%,1.5\%,1.7\%$ on METABRIC and $1.5\%,1.9\%,1.6\%$ on SUPPORT over the best baselines.

\subsection{Study of \method in Competing Events  (RQ2)}
We compare \method with baselines on \textbf{SEER} data. Results are reported in Table \ref{tab:ctd_compete}. N/A in the table means the method does not converge for the corresponding event analysis. Through experiments, we find:
\begin{itemize}[leftmargin=*, itemsep=0pt, labelsep=5pt]
\item In competing events scenarios, the performance for rare events is much worse than for common events while \method works better on HD and BC than baselines. We credit it to (1) IPS offers an unbiased estimate of objectives hence assisting in balancing predictions for events. In contrast, without IPS, \method has around 0.1 reduction of $C^{\text{td}}$ on HD on the $25\%$ horizon; (2) MTL further enhances the generalizability of representations thus yielding a better performance for the target task.

\item Except for CS-CPH, other baselines have similar performance for both HD and BC. \method wins over baselines by a significant margin: it reaches 0.797, 0.788, 0.775 for HD on each time horizon, respectively, which are $3.0\%, 2.5\%, 1.2\%$ better than the best baselines. This demonstrates the usability of the cutting-edge deep learning techniques for gaining improvement for survival analysis. More importantly, comparing to the reduced version, \method with MTL and IPS gets the best performance, which signifies the need of considering selection bias in competing events and using MTL to make the best of the data.
\end{itemize}

\subsection{In-depth Analysis (RQ3 \& RQ4)}

\subsubsection{Multi-task Learning}
We analyze the utility that multi-task learning can add, results are shown on the bottom row of Tables \ref{tab:ctd_single} \& \ref{tab:ctd_compete}. It illustrates that MTL suffices to improve the training across all datasets. Specifically, on SEER, \method achieves much better performance than \method on the rare event HD. We owe this to the auxiliary mortality prediction and time prediction tasks which enhance representation learning for transformers, thus utilizing the survival data more sufficiently.

To study the contribution of each task, we conduct an ablation study, shown by Fig. \ref{fig:mtl_ablation}. We identify that both auxiliary tasks strengthen the model performance, as the model w/o MTL has the worst $C^{\text{td}}$. It validates the effectiveness of MTL in enhancing representation learning of \method. Note that the MP task renders more gain than the LS task, which shows that the MP task is more relevant to survival analysis. It tells that designing auxiliary tasks similar to major task benefit the model more.

\begin{figure}[t]
    \centering
    \includegraphics[width=0.48\textwidth]{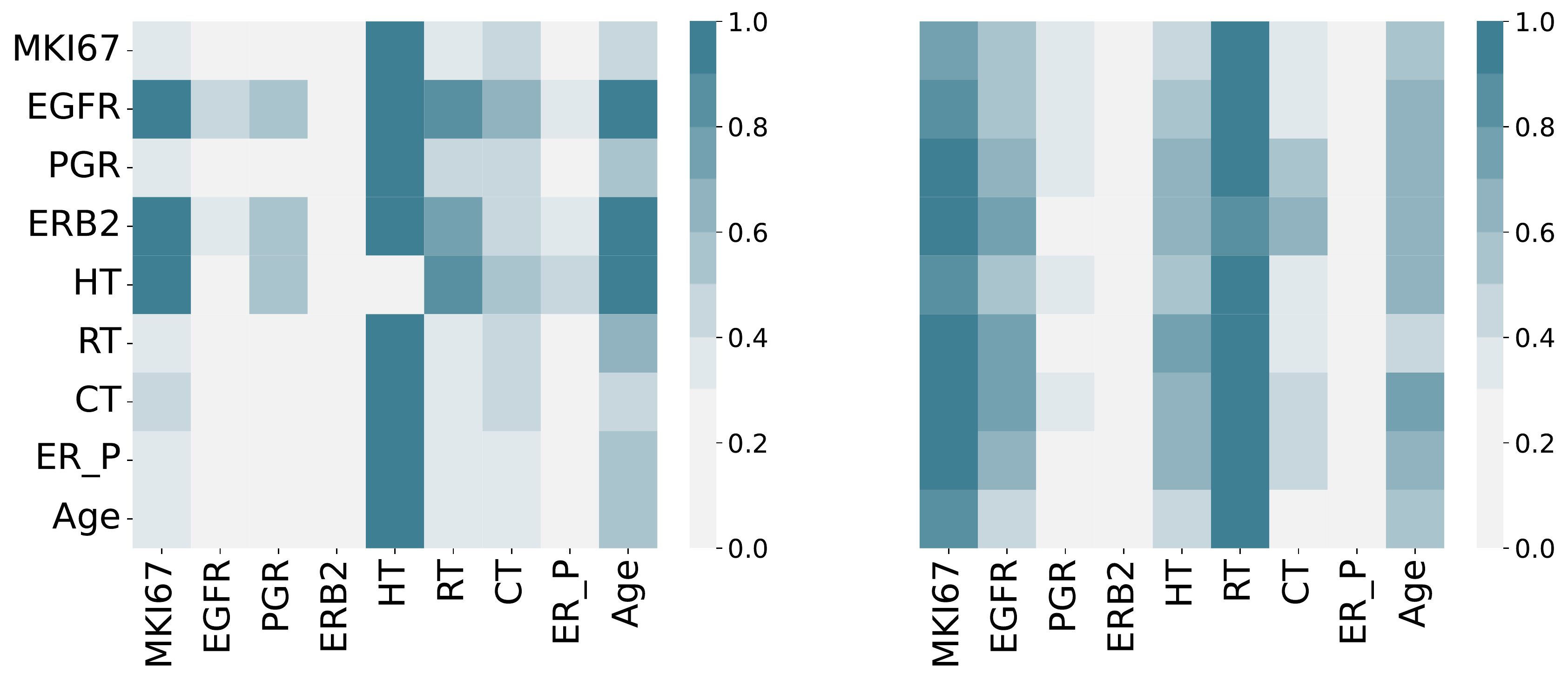}
    \caption{Visualization of attention scores between covariates for two patients; MKI67, EGFR, PGR, ERB2, and ER, are gene biomarkers; HT: Hormone treatment; RT: Radiotherapy; CT: Chemotherapy; ER\_P: ER Positive.}
    \label{fig:case_att}
\end{figure}

\begin{figure}[t]
  \begin{subfigure}[b]{0.23\textwidth}
    \includegraphics[width=\textwidth]{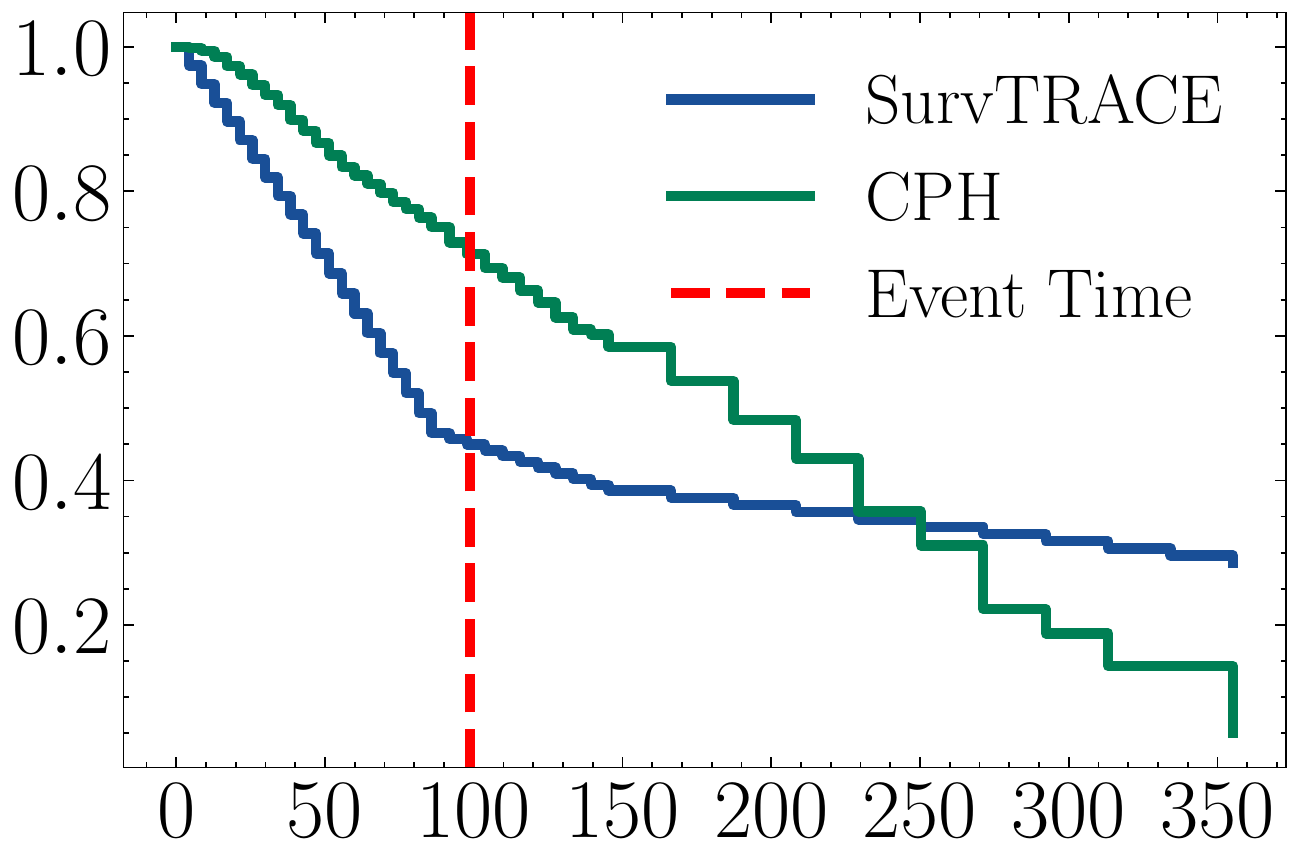}
    \caption{Predicted survival function for \textbf{uncensored} data. \label{fig:vis_event}}
  \end{subfigure}
  \hfill
  \begin{subfigure}[b]{0.23\textwidth}
    \includegraphics[width=\textwidth]{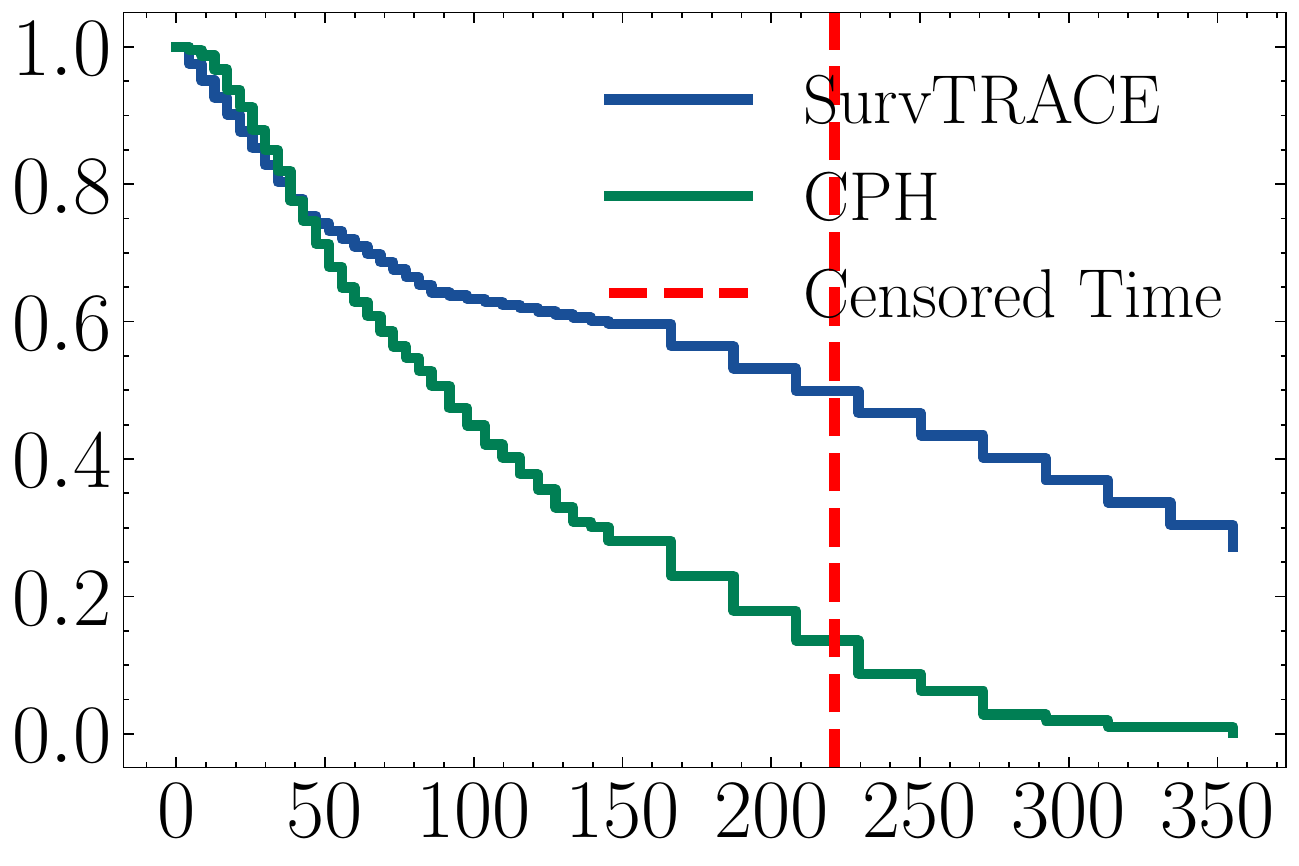}
    \caption{Predicted survival function for \textbf{censored} data.  \label{fig:vis_censored}}
  \end{subfigure}
\caption{Predicted survival function for an individual (by interpolation) by \method and CPH, respectively. $x$-axis shows the duration time points; $y$-axis shows the probability. Dotted lines stand on the point where the event/censoring happens. \label{fig:vis_case_study}}
%   \vskip -1.5em
\end{figure}

\subsubsection{Interpretability} We investigate the attention scores between different covariates. Samples of two patients are shown in Fig. \ref{fig:case_att}. The first patient (on the left) takes hormone treatment. Likewise, the HT indicator shows a significant correlation to almost all the rest covariates, which makes it an important factor for predicting the outcomes. The second patient (on the right) takes radiotherapy, and we observe the same degree of saliency in the case. For both two patients, the treatment indicator is deeply correlated to their age, pointing out the influence of age on the effectiveness of specific therapies.

An interesting finding is that the ER-positive indicator seems to have a low effect on other factors. About $85\%$ of all breast cancers are Estrogen Receptor (ER) positive, which means most patients have the same value for this term. The model downweighs the attention over it because it does not offer much additional information for discriminating survival probability across patients. Moreover, we identify the MKI67 biomarker plays a significant role in other factors, including EGFR and ERB2.

\subsubsection{Case Study}
To better visualize the superiority of \method, we plot examples of predicted survival functions by \method and CPH, respectively (Fig. \ref{fig:vis_case_study}). We identify that for uncensored data, our model senses the event happening, and the predicted probability decreases sharply. On the contrary, CPH fails to capture the signal of events and gives relatively even hazard prediction across the whole period. On the other hand, for the censored data, \method also succeeds in maintaining a high survival probability before the censored time point.

\section{Conclusion}
In summary, we propose a multi-task transformer-based survival analysis network, namely \method, which can handle both censored data and competing risks. Specifically, we take the implicit bias in censoring survival data into account and propose to debias through counterfactual learning. We also design two auxiliary tasks to utilize limited survival data for representation learning of \method. According to the visualization of the attention module engaged in \method, we can provide a case-by-case explanation for each individual. Our future work will further take time-varying covariates and multimodal data into consideration for enhancing survival analysis.

\section*{Acknowledgement}
This work was supported by NSF award SCH-2014438, IIS-1838042, NIH award R01 1R01NS107291-01 and OSF Healthcare.

%%
%% The next two lines define the bibliography style to be used, and
%% the bibliography file.
\bibliographystyle{ACM-Reference-Format}
\bibliography{sample-base}

%%
%% If your work has an appendix, this is the place to put it.
\clearpage

\appendix

\setcounter{equation}{0}

\section{Unbiased Estimate using IPS-based Loss}\label{sec:appx_lh_ips}
\begin{proposition}[IPS Loss is Unbiased Estimate of True Loss]
The proposed inverse propensity score based logistic hazard loss $L_{\text{IPS}}$ is the unbiased estimate of true risk $L^*$, as
\begin{equation}
    L^* \triangleq \frac1{nK_E} \sum_{i,k} \ell_{ik}.
\end{equation}
\end{proposition}
\begin{proof}
Let's first see why the naive LH loss is biased with selection bias censoring. The naive LH loss is defined by
\begin{equation}
    L_{\text{naive}} = \frac1{\sum_{i,k} \mathbbm{1}_{ik}} \sum_{i,k}  \mathbbm{1}_{ik} \ell_{ik}.
\end{equation}
If we take the expectation of it on the event indicator $e$, we shall get
\begin{equation}
\begin{aligned}
        \mathbb{E}_e[L_{\text{naive}}] & = \frac1{\sum_{i,k} \mathbbm{1}_{ik}} \sum_{i,k}   \mathbb{E}_e[\mathbbm{1}_{ik}] \ell_{ik} \\
        & =  \frac1{\sum_{i,k} \mathbbm{1}_{ik}} \sum_{i,k}   \pr(e_i=k) \ell_{ik}\\  & \neq  L^*.
\end{aligned}
\end{equation}
This is because the censoring of an event is not at random, and the probability of event occurrences depends on the patient's characteristics, which are so-called \emph{confounders} under the context of counterfactual learning. Therefore, we build a new estimator $L_{\text{IPS}}$ as
\begin{align}
\mathbb{E}_e[L_{\text{IPS}}] & = \frac1{nK_E} \sum_{i,k} \frac{\mathbb{E}_e[\mathbbm{1}_{ik}] \ell_{ik}}{\pi_{ik}} \\
& = \frac1{nK_E} \sum_{i,k} \frac{\pr(e_i=k) \ell_{ik}}{\pi_{ik}} \\
& = \frac1{nK_E} \sum_{i,k} \ell_{ik} \\
& = L^*.
\end{align}
In detail, we require an IPS estimator to obtain $\pi_{ik}$ due to the \emph{observational} setting: the patients are part of the assignment mechanism that generates the observational matrix. In other words, the appearance of events is covariate-dependent. Hence, we ought to estimate the propensity score $\pi$ from the observational matrix ourselves, as done by Eq. \eqref{eq:IPS_estimator}.
\end{proof}

\end{document}